\documentclass[11pt]{article}
\usepackage[T1]{fontenc}

\usepackage[margin=1in]{geometry}
\usepackage{placeins}
\usepackage{subcaption}

\usepackage{graphicx}
\usepackage{authblk}
\usepackage{color}
\usepackage{doi}
\usepackage{amsmath, amssymb, amsthm} 
\newtheorem{remark}{Remark}
\usepackage{caption}   
\usepackage{hyperref}
\usepackage{tikz}
\usepackage{braket}
\usetikzlibrary{arrows.meta,backgrounds}

\pgfdeclarelayer{nodelayer}
\pgfdeclarelayer{edgelayer}
\pgfsetlayers{background,main,nodelayer,edgelayer}

\newtheorem{definition}{Definition}
\newtheorem{proposition}{Proposition}
\newtheorem{corollary}{Corollary}
\newtheorem{theorem}{Theorem}
\begin{document}
\title{Analysis of Dirichlet Energies as Over-smoothing Measures}

\author{
  Anna Bison and Alessandro Sperduti
}
\date{}
\affil{Department of Mathematics ``Tullio Levi-Civita''\\ University of Padova, Padova, Italy \\ 

\texttt{anna.bison@phd.unipd.it, alessandro.sperduti@unipd.it}}
\maketitle             
\begin{abstract} We analyze the distinctions between two functionals often used as over-smoothing measures: the Dirichlet energies induced by the unnormalized graph Laplacian and the normalized graph Laplacian. We demonstrate that the latter fails to satisfy the axiomatic definition of a node-similarity measure proposed by Rusch \textit{et al.} By formalizing fundamental spectral properties of these two definitions, we highlight critical distinctions necessary to select the metric that is spectrally compatible with the GNN architecture, thereby resolving ambiguities in monitoring the dynamics. \end{abstract}

\newcommand{\keywords}[1]{\noindent\textbf{Keywords:} #1}

\section{Introduction}
One of the most analyzed problems in Graph Neural Networks (GNNs) is over-smoothing, that is usually described as the exponential convergence of node embeddings to a common vector through the GNN layers. One of the more frequently used metrics to analyze both theoretically and empirically over-smoothing is the Dirichlet energy, that is induced by the graph Laplacian, with different possibilities as analyzed in the next section. A formal axiomatic definition of over-smoothing, based on the definition of a "total over-smoothing" state where all node embeddings are identical, has been proposed in \cite{Rusch2023ASO}. A key axiom of the proposal is that a smoothness measure should be zero if and only if this state is reached. In this paper, we point out that the widely-used Dirichlet Energy induced by the \textit{normalized} graph Laplacian does not satisfy this axiom. Recently, some other issues in adopting Dirichlet energies in order to measure over-smoothing were pointed out in \cite{zhang2025rethinkingoversmoothinggraphneural}, where it is explained that Dirichlet energy induced by the normalized Laplacian tends to zero when node embeddings tend to its dominant eigenvector $\vec{v}$ s.t. $v_i = \sqrt{d_i +1}$, with $d_i$ the degree of the $i$-th node. The difference between our work and \cite{zhang2025rethinkingoversmoothinggraphneural} is that while Zhang \textit{et al.} also use this convergence to argue that the Dirichlet metric is too strict as a definition (motivating their rank-based proposal), they do not focus on the fact that this represents a formal violation of the \cite{Rusch2023ASO} axiom. We, instead, center our analysis on this specific axiom violation. Our primary objective is to demonstrate that the diagnostic validity of Dirichlet energy is not absolute, but strictly conditional on its consistency with the operator governing the GNN dynamics. We use the distinction between the standard Graph Laplacian, $\Delta$, and its normalized counterpart, $\Delta_{\text{norm}}$, as a paradigmatic example of this broader principle.
While often treated as interchangeable, we show that measuring normalized dynamics with unnormalized metrics leads to a breakdown in detection. In this incompatible regime, the standard energy fails to capture genuine over-smoothing (convergence to the normalized kernel); conversely, observing a vanishing energy becomes a misleading indicator: since the dynamics do not drive the signal to the metric's kernel, such decay strictly implies the collapse of embedding norms (over-shrinking), rather than geometric alignment.
A pertinent example is found in recent work \cite{arroyo2025vanishinggradientsoversmoothingoversquashing}, where the analysis of normalized models using the standard energy leads to the conclusion that over-smoothing is driven by contractive layers. As we argue, this is an artifact of the metric mismatch. While Johnson and Zhang \cite{DBLP:journals/corr/abs-1810-00826} previously highlighted the bias introduced by normalization in standard semi-supervised learning, our work investigates the spectral consequences of this discrepancy specifically within the GNN framework. We analyze how the non-commutativity between the filter and the target metric generates spectral dispersion, thereby invalidating standard approaches to monitoring over-smoothing.

\section{Notation}
We denote vectors with arrows $\vec{v}$ and their $i$-th entry with $v_i$. We denote the standard basis vectors as $e_i$ and the all-ones vector as $\mathbf{1} = \sum_{i=1}^N e_i$, with $N$ the dimension of the vector space. We denote the standard inner product between two vectors $\vec{u}, \vec{v}$  with the Bra-ket notation as $\braket{u|v}$. In addition, we denote by $\delta_{ij}$ the Kronecker delta, defined as $\delta_{ij}=1$ if $i=j$ and $0$ otherwise.
We denote matrices with uppercase letters (e.g., $X$), the vector associated to their $i$-th row with $X_i^T$ and their entries with $X_{ij}$. Throughout this paper, we consider $\mathcal{G} = (\mathcal{V}, \mathcal{E})$ to be an undirected and connected\footnote{This assumption simplifies the spectral analysis by ensuring that the eigenvalue $\lambda_0 = 0$ of the Laplacian has multiplicity 1.} graph, we denote with $\mathcal{V}$ the set of nodes, where $\mathcal{V}=\{1, \dots, |\mathcal{V}|\}$ and with $ \mathcal{E} \subseteq \mathcal{V} \times \mathcal{V}$ the set of edges. We indicate an edge between node $i$ and $j$ as $(i, j) \in  \mathcal{E}$.
The graph is associated with a node feature matrix $X \in \mathbb{R}^{|\mathcal{V}| \times m}$, where the $i$-th row represents the feature vector of node $i$, and the $j$-th column represents a scalar graph signal corresponding to the $j$-th feature.
Let $A \in \{0,1\}^{|\mathcal{V}| \times |\mathcal{V}|}$ be the adjacency matrix, where $A_{ij}=1$ if $(i, j) \in \mathcal{E}$ and $0$ otherwise, and let $D = \text{diag}(d_1, \dots, d_{|\mathcal{V}|})$ be the degree matrix. We denote a constant scalar signal of value $1$ for all the nodes with $\mathbf{1} \in \mathbb{R}^{|\mathcal{V}|}$ and the set of neighbors of node $i$ as $\mathcal{N}_i$.
Finally, we define the unnormalized Laplacian as $\Delta = D - A$ and the symmetric normalized Laplacian as $\Delta_{\text{norm}} = I - D^{-1/2}AD^{-1/2}$.

\section{Axiomatic Formulation}

Consider the definition of node-similarity given in Rusch \textit{et al.} \cite{Rusch2023ASO}:

\begin{definition}{\textbf{Over-smoothing} \cite{Rusch2023ASO}}
    \\Let \( \mathcal{G} \) be an undirected, connected graph and let \( X^{(n )}\in \mathbb{R}^{|\mathcal{V}| \times m} \) denote the \( n \)-th layer hidden features of an \( N \)-layer GNN defined on \( \mathcal{G} \). Moreover, we call \( \mu : \mathbb{R}^{|\mathcal{V}| \times m} \to \mathbb{R}_{\geq 0} \) a node-similarity measure if it satisfies the following axioms:  
\begin{enumerate}

 \item  \( \exists \; \vec{c} \in \mathbb{R}^m \) with \( X_i^T = \vec{c} \) for all nodes \( i \in V \) \( \Leftrightarrow \) \( \mu(X) = 0 \), for \( X \in \mathbb{R}^{|\mathcal{V}| \times m} \)  
 \item  \( \mu(X + Y) \leq \mu(X) + \mu(Y) \), for all \( X, Y \in \mathbb{R}^{|\mathcal{V}| \times m} \)  \\
 
Over-smoothing with respect to \( \mu \) is then defined as the layer-wise exponential convergence of the node-similarity measure \( \mu \) to zero, i.e., 

 \item  \( \mu(X^{(n )}) \leq C_1 e^{-C_2 n}, \) for \( n = 0, \dots, N \) with some constants \( C_1, C_2 > 0 \).
 \end{enumerate}
\end{definition}
One example of the most common over-smoothing measure used in literature is the Dirichlet energy. 
It can be defined as a function of the Laplacian, and depending on the definition of the Laplacian chosen, it can happen that it could not induce a node-similarity measure with respect to the previous definition. 
\begin{definition}
Let $\mathcal{G}(\mathcal{V}, \mathcal{E})$ be a graph and $X \in \mathbb{R}^{|\mathcal{V}| \times m}$ be the matrix of the node embeddings associated to $\mathcal{G}$.
The following definitions of Dirichlet energy are considered:
    \begin{itemize}
        \item from \cite{Rusch2023ASO}, induced by $\Delta$
            \[
             \mathcal{E}_{\Delta}(X) = \frac{1}{|\mathcal{V}|}\text{tr}(X^{\top}\Delta X)
            \]
            \[
=\frac{1}{|\mathcal{V}|}\sum_{i \in \mathcal{V}}\sum_{j \in \mathcal{N}_i}  \left\| X_i^T - X_j^T \right\|^2_2 ;
            \]
        \item induced by $\Delta_{\text{norm}}$
            \[
             \mathcal{E}_{\Delta_{\text{norm}}}(X) = \text{tr}(X^{\top}\Delta_{\text{norm}} X)
            \]
            \[
= \sum_{i \in \mathcal{V}}\sum_{j \in \mathcal{N}_i}  \left\| \frac{X_i^T}{\sqrt{ d_i}} - \frac{X_j^T}{\sqrt{ d_j}} \right\|^2_2 ;
\]
            
        \item induced by
            $\tilde{\Delta}_{\text{norm}}:=I - (D +I)^{-1/2}(A+I)(D +I)^{-1/2}$, version with self-loops considered in \cite{Cai2020ANO}, induced by the renormalization trick introduced in \cite{kipf2017semisupervisedclassificationgraphconvolutional}
            \[
             \mathcal{E}_{\tilde{\Delta}_{\text{norm}}}(X) = \text{tr}(X^{\top}\tilde{\Delta}_{\text{norm}} X)
            \]
            \[
= \sum_{i \in \mathcal{V}}\sum_{j \in \mathcal{N}_i}  \left\| \frac{X_i^T}{\sqrt{1 + d_i}} - \frac{X_j^T}{\sqrt{1 + d_j}} \right\|^2_2 .
\]

    \end{itemize}
\end{definition}
\begin{remark}
Notice that, in the expression of $\mathcal{E}_{\Delta}$, the constant $\frac{1}{|\mathcal{V}|}$ was added in order to express the definition of $\mathcal{E}$ reported in \cite{Rusch2023ASO}. In subsequent analyses, this constant will be excluded because it does not impact the results for the intended applications and is typically introduced to stabilize the value ranges of Dirichlet energy, allowing for the comparison of energies across graphs with varying node counts.
\end{remark}
As reported in \cite{Rusch2023ASO}, $\sqrt{\mathcal{E}_{\Delta}}$ satisfies the definition of a node-similarity measure. Nevertheless, there are conditions of exponential convergence of the node embeddings to a subspace of dimension $1$ that, according to the spectral concept of smoothing, should be considered conditions of over-smoothing, that instead the Definition 1 does not take into account, and for which, accordingly, $\sqrt{\mathcal{E}_{\Delta}}$ is not a precise metric. One example is the exponential convergence of node embeddings to the kernel of the normalized (or analogously augmented with self-loops and normalized) Laplacian, for which the eigenspaces are different w.r.t. the ones of $\Delta$, and for which the best-suited metric should be the normalized Dirichlet energy $\mathcal{E}_{\Delta_{\text{norm}}}$ (or respectively, $\mathcal{E}_{\tilde{\Delta}_{\text{norm}}}$).\\
Therefore, it seems worth pointing out that $\mathcal{E}_{\Delta_{\text{norm}}}$ does not satisfy the first axiom of node-similarity measure. Indeed, when a graph is \textbf{not regular}\footnote{If a graph is regular, both $\mathcal{E}_{\Delta_{\text{norm}}}$ and $\mathcal{E}_{\tilde{\Delta}_{\text{norm}}}$ are proportional to $\mathcal{E}_{\Delta}$, hence they satisfy the first axiom.}, the required implication does not hold: \[
\forall \; i \in \mathcal{V}  \;  X_i^T = \vec{c} \Rightarrow 
\mathcal{E}_{\Delta_{\text{norm}}}(X) =  \sum_{i \in \mathcal{V}}\sum_{j \in \mathcal{N}_i}  \left\| \frac{\vec{c}}{\sqrt{ d_i}} - \frac{\vec{c}}{\sqrt{ d_j}} \right\|^2_2
\]
\[
= \sum_{i \in \mathcal{V}}\sum_{j \in \mathcal{N}_i}  \|\vec{c}\|^2_2 \left( \frac{1}{\sqrt{ d_i}} - \frac{1}{\sqrt{ d_j}} \right)^2 > 0 \quad \text{if} \; \exists \; \{i, j\} \in \mathcal{E} \; | \; d_i \neq d_j .
\]
Hence, Dirichlet energy induced by the normalized Laplacian does not constitute a node-similarity measure, in particular it's not a measure of over-smoothing according to the definition given above.
An analogous result holds for $\mathcal{E}_{\tilde{\Delta}_{\text{norm}}}$.\\

An objection might be tentatively based on the idea that, as noted by the paper asserting the axiomatic definition, the first axiom holds up to a constant, meaning that a node similarity measure could satisfy the double implication converging to any other constant different than zero. 
Thus, we could consider the possibility of reformulating a node-similarity measure by recasting $\mathcal{E}_{\Delta_{\text{norm}}}$ up to a constant, in a certain manner. But this would not solve the issue, since considering for example a constant scalar signal \( c \mathbf{1}\in \mathbb{R}^{|\mathcal{V}|}, \; c \in \mathbb{R} \), the energy value would be different depending on $c$, and the definition asks for the convergence to a unique constant for all the constant signals independently from their value.\\
Indeed, if a node similarity measure would be allowed to assume different constant values depending on $c$ then, considering the $n$-th layer hidden features of an $N$-layer GNN applied to $\mathcal{G}$, and considering in input a scalar signal $X^{(0)}_i=c_0$ for all nodes $i$, in general it could be $\mu(X^{(0)})\neq \mu(X^{(n)})$, even if for each iteration it would be $X^{(n)}_i=c_n$ for all nodes $i$ (possibly different from $c_0$ depending on the function implemented by the GNN layer). But this would make it completely vacuous the notion of convergence of the node-similarity measure through the $N$ iterations. That is, such a node similarity would not be able in general to measure the over-smoothing defined in the Definition 1. Hence, the previous would not be a good objection, and then Definition 1 is formally excluding $\mathcal{E}_{\Delta_{\text{norm}}}$ as an over-smoothing metric.\\ 

The two Laplacians behave differently w.r.t. the first axiom because the kernel of the normalized Laplacian does not coincide in general with the kernel of the standard Laplacian, which is the set of smooth signals proportional to $\mathbf{1} \in \mathbb{R}^{|\mathcal{V}|}$ that is independent by the topology of the graph.\\
The kernel of $\Delta_{\text{norm}}$ is instead generated by $D^{1/2}\mathbf{1}$, that depending on $D^{1/2}$, and then on the specific graph, is spread along different eigenspaces of $\Delta$. \\

A formalization of the above discussion is given in the following in the context of the GNNs domain. In particular, all the obtained formal results are consequences of the following straightforward mathematical fact:

\begin{theorem} In general, $\Delta_{\text{norm}}$ and $\Delta$ do not commute, and therefore they are not simultaneously diagonalizable.
\end{theorem}
\begin{proof}
To show that the two operator do not commute, it has to be shown that $[\Delta_{\text{norm}}, \Delta] = \Delta_{\text{norm}}\Delta - \Delta\Delta_{\text{norm}}\neq 0_{|\mathcal{V}|\times |\mathcal{V}|}$. Hence, it's sufficient to show that $\exists$ $\vec{v} \in \mathbb{R}^{|\mathcal{V}|} \setminus \{\vec{0} \}$ s.t. $[\Delta_{\text{norm}}, \Delta]\vec{v} \neq \vec{0}$. It's straightforward to see that $\mathbf{1}$ satisfies that condition (notice that $\Delta \mathbf{1}=diag(A\mathbf{1})\mathbf{1}-A\mathbf{1}= \vec{0} $):
    \[
[\Delta_{\text{norm}}, \Delta]  \mathbf{1} =  \Delta_{\text{norm}}\Delta \mathbf{1} -\Delta\Delta_{\text{norm}}\mathbf{1} =-\Delta \Delta_{\text{norm}}\mathbf{1} 
\]
\[
=\Delta \sum_{i = 1}^{|\mathcal{V}|} \biggl(- 1+\sum_{j \in \mathcal{N}_i}\frac{1}{\sqrt{d_i d_j}} \biggr) e_i \neq \vec{0} 
\]
since $\Delta \vec{v} = \vec{0} \iff \exists \; c \in \mathbb{R} \; | \; \vec{v}=c \mathbf{1}$, and in general \\ $(- 1+\sum_{j \in \mathcal{N}_i}\frac{1}{\sqrt{d_i d_j}} ) \neq (- 1+\sum_{j \in \mathcal{N}_k}\frac{1}{\sqrt{d_k d_j}}  )$ for $k \in \mathcal{V} \setminus \{i\}$.
\end{proof}
The non-commutativity between the two Laplacians arises from the non-commutativity between $D$ and $A$. By linearity and the Leibniz rule, we have:
\[
[\Delta_{\text{norm}}, \Delta] =[I-D^{-1/2}AD^{-1/2}, D-A]=-[D^{-1/2}AD^{-1/2}, D] + [D^{-1/2}AD^{-1/2},A]
\]
\[
=-D^{-1/2}[A,D]D^{-1/2}+[D^{-1/2},A]AD^{-1/2}+D^{-1/2}A[D^{-1/2},A]
\]
which generally is $\neq 0_{|\mathcal{V}| \times |\mathcal{V}|}$, since $[D,A] = 0_{|\mathcal{V}| \times |\mathcal{V}|} \iff [D^{-1/2}, A]=0_{|\mathcal{V}| \times |\mathcal{V}|}$. A straightforward example of $[A, D] \neq  0_{|\mathcal{V}| \times |\mathcal{V}|}$ is given by a four-node graph consisting of a triangle and an additional edge connecting the fourth node to one of the triangle's vertices:
    \[
A=\begin{pmatrix}
0 & 1 & 1 & 1\\
1 & 0 & 1 & 0\\
1 & 1 & 0 & 0\\
1 & 0 & 0 & 0
\end{pmatrix},\qquad
D=\begin{pmatrix}
3 & 0 & 0 & 0\\
0 & 2 & 0 & 0\\
0 & 0 & 2 & 0\\
0 & 0 & 0 & 1
\end{pmatrix}
\]

\[
AD=\begin{pmatrix}
0 & 2 & 2 & 1\\
3 & 0 & 2 & 0\\
3 & 2 & 0 & 0\\
3 & 0 & 0 & 0
\end{pmatrix} \neq \begin{pmatrix}
0 & 3 & 3 & 3\\
2 & 0 & 2 & 0\\
2 & 2 & 0 & 0\\
1 & 0 & 0 & 0
\end{pmatrix}= DA 
\]
Thus, $A$ and $D$ do not commute.

    A specific set of graphs for which $[A, D]=0_{|\mathcal{V}| \times |\mathcal{V}|}$ is regular graphs s.t. $D= dI$, since $[I,A] = 0_{|\mathcal{V}| \times |\mathcal{V}|}  \; \forall A$. Indeed, as shown previously, for regular graphs both $\Delta$ and $\Delta_\text{norm}$ induce a Dirichlet energy that satisfies the axiom of node similarity measure. \\
    
One effect of the previous theorem is the following:
\begin{proposition}
Let \( \mathcal{G} \) be an undirected, connected and \textbf{non-regular} graph and let \( X^{}\in \mathbb{R}^{|\mathcal{V}| \times m} \) denote the hidden features of the signal defined in \( \mathcal{G} \).\\
Let $M_{\mu} \in \mathbb{C}^{|\mathcal{V}| \times |\mathcal{V}|}$ be any operator that induces a node-similarity measure in the following way 
\[\mu(X):= \text{tr}(X^{\top}M_{\mu}X) \geq 0.\] 
Then
\[
 D^{-1/2}M_{\mu}D^{-1/2}
\]
cannot induce a node-similarity measure.
\end{proposition}
\begin{proof}
 Consider a scalar signal $X_i= \frac{c}{\sqrt{d_i}}\;, c \in \mathbb{R}\;, \forall \; i \in \mathcal{V}$, then $X$ can be rewritten as $D^{-1/2 }c\mathbf{1} \in \mathbb{R}^{|\mathcal{V}| }$. Then, by the first axiom of node-similarity measure, $\mu$ must be strictly positive in $X$, since $X_i$ is not the same constant scalar $c$ for all nodes $i$. In particular:
 \[
\mu(X)= \mu(D^{-1/2}c\mathbf{1}) = \text{tr}(c\mathbf{1}^{\top} D^{-1/2}M_{\mu}D^{-1/2}c\mathbf{1}) = \mu'(c\mathbf{1})>0 .
 \]
 i.e., the previous expression has a double read: it can be considered also as a different node similarity measure $\mu'$, induced by $D^{-1/2}M_{\mu}D^{-1/2}$, but applied to $c\mathbf{1}$, that instead is a constant signal through all the nodes $i$. That is, the node-similarity measure $\mu'$ induced by $D^{-1/2}M_{\mu}D^{-1/2}$ should have taken in $c\mathbf{1}$ the same value that $\mu$ takes in $D^{-1/2}c\mathbf{1}$.\\
 But if $D^{-1/2}M_{\mu}D^{-1/2}$ could have induced such a node-similarity measure, then it should have satisfied the first axiom. 
 Hence, it should have been that, considering the null signal $X_i= 0 \; \forall \; i \in \mathcal{V}$:
 \[
\text{tr}(\vec{0}^{\top} D^{-1/2}M_{\mu}D^{-1/2}\vec{0}) = 0 .
\]
 In particular, it would have meant that the node-similarity measure induced by $D^{-1/2}M_{\mu}D^{-1/2}$ would have assumed different values depending on different possible $c$ s.t. $X_i= c \; \forall \; i \in \mathcal{V}$.
Then, a node similarity measure induced by $D^{-1/2} M_{\mu}D^{-1/2}$ could not satisfy the first axiom, from which follows the thesis.
\end{proof}
Notice that the previous proof relies on the fact that, generally, $M_{\mu}$ does not commute with its normalized counterpart $M'_{\mu} = D^{-1/2}M_{\mu}D^{-1/2}$. Crucially, the global minimum of the node-similarity measures $\mu$ and $\mu'$ is achieved if and only if the vector lies in the kernel of the respective operator: $\mu(\vec{x}) = 0 \iff \vec{x} \in \text{Ker}(M_{\mu}), \quad \text{and} \quad \mu'(\vec{x}) = 0 \iff \vec{x} \in \text{Ker}(M'_{\mu})$. Since the operators do not commute (under the assumption that the graph is connected and non-regular), they are not simultaneously diagonalizable and their kernels define distinct one-dimensional subspaces. Consequently, there exists no vector in $\mathbb{R}^{|\mathcal{V}|} \setminus \{\vec{0}\}$ where $\mu$ and $\mu'$ are simultaneously zero, and the previouses double implications are never simultaneously satisfied in $\mathbb{R}^{|\mathcal{V}|} \setminus \{\vec{0}\}$.
\begin{remark}
   Observe that the earlier results remain valid under the application of a scalar positive semidefinite operation, such as the square root:  \[
   \mu(X):= \sqrt{\text{tr}(X^{\top}M_{\mu}X) } \geq 0
   \] 
   as the preceding proof remains unchanged.
\end{remark}
In particular:
\begin{corollary}
The normalized Laplacian $\Delta_{\text{norm}}$ does not induce a node-similarity measure. 
\begin{proof}
    From \cite{Rusch2023ASO} it's known that $\sqrt{\mathcal{E}_{\Delta}}$ is a node-similarity measure induced by $\Delta$.
    From the previous proposition it follows that $D^{-1/2}\Delta D^{-1/2}= \Delta_{\text{norm}}$ cannot induce a node-similarity measure. 
\end{proof}
\end{corollary}

\begin{remark}
    All the previous theorems hold also for the normalized Laplacian augmented with self loops, $\tilde{\Delta}_{\text{norm}}$, since 
    \[
    \tilde{\Delta}_{\text{norm}} = I-(D+I)^{-1/2}(A+I)(D+I)^{-1/2} 
    \]
    \[= (D+I)^{-1/2}\bigl( (D+I)-(A+I)\bigr)(D+I)^{-1/2}=(D+I)^{-1/2}\bigl(D-A)(D+I)^{-1/2}
    \]
    \[
    = \tilde{D}^{-1/2}\Delta \tilde{D}^{-1/2}.
    \]
    in particular
    \[
    \Delta\mathbf{1}= \vec{0}\Rightarrow\tilde{\Delta}_{\text{norm}}\tilde{D}^{1/2}\mathbf{1} =\tilde{D}^{-1/2}\Delta\mathbf{1} = \vec{0}.
    \]
    Hence, all the previous theorems can be shown for $\tilde{\Delta}_{\text{norm}}$ too, applying in the proofs the substitution $d_i \mapsto d_i + 1$ $\forall \; i \in \mathcal{V}$ from which it follows that $D \mapsto D+I = \tilde{D}$ and $A \mapsto A+I = \tilde{A}$.
    
\end{remark}

Hence, analysis of over-smoothing of GCNs augmented with self-loops should be based on the Dirichlet energy definition induced by $\tilde{\Delta}_{\text{norm}}$, since it measures the smoothness with respect to the Laplacian that governs the evolution of graph signals, i.e. the collapse of the graph signals of the features to the kernel of $\tilde{\Delta}_{\text{norm}}$, that represents the condition where node embeddings are a rescaling of $\sqrt{d_i+1}$, i.e., where all the information about data is lost and the embedding matrix only encodes information about the degree-distribution (i.e. losing the ability of distinguishing graphs representing data of different nature but with the same topology). This loss of information is precisely the one that gave raising to the problem of over-smoothing, since it's precisely the kind of smoothness GCNs are prone to: the simplest definition of a layer of a GCN applied to an embedding matrix $X$ is $\sigma(\tilde{A}XW) = \sigma \bigl( (I-\tilde{\Delta}_{\text{norm}})XW \bigr)$ that, when
$\sigma$ is the identity function, gives rise to a low pass filter in a connected graph. A spectral smoothness known and managed well before the first definition of GCN, as shown in the fairing algorithm defined in \cite{Taubin1995ASP}, \cite{466848}, that has a straightforward analysis in the framework of GNNs, as shown in \cite{spectralinterpretation}. For this reason, the Dirichlet energy induced by the normalized Laplacians should be included in a well-posed axiomatic definition of over-smoothing. 
In other words, a problem with Definition 1 is that it's based only on the embedding matrices and it is completely independent from the topology of the graph. Indeed, axiom 1 is only mentioning a condition regarding the loss of information contained on the embedding matrix $X$, meaning that it's imposing a constraint holding for all the topologies having the same embedding matrix $X$, possibly excluding the smoothness arising from a topology dependent node-similarity, as in the case formally analyzed in Proposition 1.
In other words, Definition 1 is formally excluding the spectral smoothness w.r.t. normalized Laplacians from the over-smoothing class of phenomena, hence is too strict. 
Although this is a simple mathematical fact, it is useful to formally and precisely highlight the subtle difference it induces in the interpretation of results, particularly in the context of over-smoothing in GNNs.

\section{Indeterminacy}

Spectral analysis is straightforward when GNN layers act as spectral filters (e.g., polynomials) in the Laplacian eigenvector basis. Yet, as shown in the previous chapter, a consistent Laplacian choice is necessary for unbiased over-smoothing analysis. Here, we consider the evolution of the features matrix of a graph through a GNN where we assume that the composition of its first $k$ layers implements a function $f_k(\Delta_{\text{norm}})$, where $f_k(\Delta_{\text{norm}})$ is a spectral function of $\Delta_{\text{norm}}$, $f_k(\xi)$ is the evaluation of $f_k$ on the scalar $\xi$, and $\overline{f_k(\xi)}$ is the complex conjugate of $f_k(\xi)$. We illustrate how this inconsistency regarding the underlying spectral dynamics could lead to misleading analytical results, evaluating  $f_k(\Delta_{\text{norm}})$ against the standard Dirichlet energy.
The conclusion holds in general for all the functions that do not commute with $\Delta$. Let $\mathcal{B}$ be any orthonormal base of the features space $\mathbb{R}^d$, $\sigma_{\Delta}$ be the spectrum of eigenvalues of $\Delta$, and  consistently for $\sigma_{\Delta_{\text{norm}}}$:
\[
             \mathcal{E}_{\Delta}(X^{(k)}) = \text{tr}(X^{(k)^\top}\Delta X^{(k)})= \text{tr}(X^{(0)^\top}f_k(\Delta_{\text{norm}})\Delta f_k(\Delta_{\text{norm}}) X^{(0)})
            \]
       \[
= \sum_{\omega \in \sigma_{\Delta}} \omega  \sum_{\vec{e} \in \mathcal{B}} \vec{e}^TX^{(0),\top}f_k(\Delta_{\text{norm}}) \vec{\omega}\vec{\omega}^T  f_k(\Delta_{\text{norm}}) X^{(0)} \vec{e}
\]
   \[
= \sum_{\omega \in \sigma_{\Delta}}\sum_{\xi \in \sigma_{\Delta_{\text{norm}}}} \sum_{\xi' \in \sigma_{\Delta_{\text{norm}}}} \omega \overline{f_k(\xi)} f_k(\xi') \sum_{\vec{e} \in \mathcal{B}} \vec{e}^TX^{(0),\top}\vec{\xi}\braket{\xi|\omega}     \braket{\omega|\xi'}\vec{\xi'}^T     X^{(0)} \vec{e}
\]
expressing now $\vec{\xi}$ in the base of $\vec{\omega}$ it's
$\vec{\xi} = I\vec{\xi}=\sum_{\omega \in \sigma_{\Delta}} \vec{\omega}\braket{\omega|\xi}$, and analogously for $\vec{\xi}^T$.  Hence, each projector of shape $\vec{\xi}\vec{\xi}^T$ becomes:
\[
\sum_{\omega \in \sigma_{\Delta}}\sum_{\omega^{\prime} \in \sigma_{\Delta}} \vec{\omega}\braket{\omega|\xi} \braket{\xi|\omega^{\prime}} \vec{\omega}^{\prime T}= \sum_{\omega \in \sigma_{\Delta}}\sum_{\omega^{\prime}\in \sigma_{\Delta}} \braket{\omega|\xi} \braket{\xi|\omega^{\prime}} \vec{\omega}\vec{\omega}^{\prime, T}
\]
substituting in the expression above, keeping in account that $\braket{\omega^{\prime}|\omega}=\delta_{\omega \omega^{'}}$
\[
\mathcal{E}_{\Delta}(X^{(t)})= \sum_{\omega \in \sigma_{\Delta}} \omega \sum_{\omega^{\prime} \in \sigma_{\Delta}} \sum_{\omega^{\prime\prime} \in \sigma_{\Delta}} \sum_{\xi \in \sigma_{\Delta_{\text{norm}}}} \sum_{\xi^{\prime} \in \sigma_{\Delta_{\text{norm}}}}
\]
\[
\overline{f_k(\xi)} f_k(\xi') \sum_{\vec{e} \in \mathcal{B}} \vec{e}^TX^{(0),\top} \vec{\omega^{ \prime}} \braket{\omega^{ \prime}|\xi} \braket{\xi|\omega}  \braket{\omega|\xi^{\prime}}   \braket{\xi^{\prime}|\omega^{\prime \prime} }\vec{\omega}^{\prime\prime,T}   X^{(0)} \vec{e}
\]
\[
= \sum_{\omega \in \sigma_{\Delta}} \omega \sum_{\omega^{\prime} \in \sigma_{\Delta}} \sum_{\omega^{\prime\prime} \in \sigma_{\Delta}} \sum_{\xi \in \sigma_{\Delta_{\text{norm}}}} \sum_{\xi^{\prime} \in \sigma_{\Delta_{\text{norm}}}}
\] 
\[
\braket{\omega^{ \prime}|\xi} \braket{\xi|\omega}  \braket{\omega|\xi^{\prime}}   \braket{\xi^{\prime}|\omega^{\prime \prime} } \overline{f_k(\xi)} f_k(\xi') \sum_{\vec{e} \in \mathcal{B}} \vec{e}^T X^{(0),\top} \vec{\omega}^{ \prime} \vec{\omega}^{\prime\prime T}     X^{(0)} \vec{e} .
\]
This expression shows that the contributions to the Dirichlet energy now depend from the scalars $\braket{\xi|\omega}$ that express the non-commutativity between $\Delta$ and $\Delta_{\text{norm}}$, in particular, \textit{they depend on} $D$.
In fact, those scalars are coefficients of the \textit{superposition matrix} between orthonormal eigen-basis of two in general non-commutable operators, and expresses "how much" they don't commute. The more they commute, the nearest are that coefficients to $\delta$.\\

The previous analysis expresses the following limitation, that is a reformulation of Theorem 1 under that perspective:
\begin{theorem}
If a GNN is defined to implement a spectral function of $\Delta_{\text{norm}}$ then in general it does not implement a spectral function w.r.t  $\Delta$.
\end{theorem}
Only in the case of a regular graph, i.e. when $\Delta_{\text{norm}}=\frac{1}{d}\Delta$, and the superposition matrix equals $I$, it's straightforward to see that the previous reduces to
\[
\mathcal{E}_{\Delta}(X^{(k)})= \sum_{\omega \in \sigma_{\Delta}} \omega\sum_{\omega' \in \sigma_{\Delta}} \sum_{\omega'' \in \sigma_{\Delta}} \overline{f_k\biggl(\frac{\omega'}{d} \biggr)} f_k\biggl(\frac{\omega''}{d}\biggr)   
\delta_{\omega\omega'}   \delta_{\omega\omega''}
\sum_{\vec{e} \in \mathcal{B}} \bra{e}X^{(0),\top}\vec{\omega'}\bra{\omega''}     X^{(0)} \vec{e}
\] 
\[
=\sum_{\omega \in \sigma_{\Delta}}  \bigl| f_k\biggl(\frac{\omega}{d}\biggr) \bigr|^2 \mathcal{E}_{\omega}(X^{(0)}).
\] 
This means that the GNN will behave like a spectral filter $f_k$ w.r.t. $\Delta$ frequencies \textit{only} when processing regular graphs. 
Hence, the frequency band of $\Delta$ that passes depends on the degree distribution of the nodes: for regular graphs, the GNN is the same filter up to a rescaling, while for other graphs, the resulting filter depends on the decomposition of $\Delta_{\text{norm}}$ frequencies in $\Delta$ frequencies, and could vary for different input graphs. Asserting that a GNN \textit{implements} a filter implies that, \textit{for all} processed graphs, it must attenuate or amplify specific frequencies of the signal consistently with respect to a fixed Laplacian definition (e.g., regardless of the degree distribution $D$). However, if the filter function $f$ is uniquely defined on the spectrum of one Laplacian, generally it cannot preserve its spectral behavior with respect to the other, where frequencies are transformed depending on the graph structure. Consider, for instance, the kernel of $\Delta_{\text{norm}}$, which is spanned by the vector $D^{1/2}\mathbf{1}$. In a graph exhibiting high degree variability (containing both high-degree and low-degree nodes), the frequency associated with this vector is zero with respect to $\Delta_{\text{norm}}$. Conversely, this same vector $D^{1/2}\mathbf{1}$ is generally not in the kernel of the unnormalized Laplacian $\Delta$ (which is spanned by $\mathbf{1}$). Instead, it appears as a superposition of various non-zero frequencies with respect to $\Delta$. In this sense, an eigen-signal of one Laplacian, associated with a specific frequency, exhibits spectral dispersion when analyzed with respect to the other. Instead of being frequency-localized, it appears scattered across multiple components of the spectrum of the other.

\subsection{Over-shrinking and Over-smoothing}

Overlooking the formal distinction between the smoothing metrics associated with different Laplacians (e.g., $\mathcal{E}_{\Delta}$ vs. $\mathcal{E}_{\tilde{\Delta}_{\text{norm}}}$) can lead to significant interpretative pitfalls. Consider, for instance, the analysis of GNN dynamics governed by $\tilde{A}_{\text{norm}} = I - \tilde{\Delta}_{\text{norm}}$. A consistent analysis would require employing the associated energy $\mathcal{E}_{\tilde{\Delta}_{\text{norm}}}$; however, relying on the standard Dirichlet energy $\mathcal{E}_{\Delta}$ as an indicator for over-smoothing introduces a fundamental methodological mismatch. As established in the previous sections, $\mathcal{E}_{\Delta}$ measures convergence specifically to the topology-independent subspace $\text{span}(\mathbf{1})$. Since the dynamics of a GCN does not drive the signal towards this state, but rather towards the topology-dependent kernel of $\tilde{\Delta}_{\text{norm}}$, the standard energy $\mathcal{E}_{\Delta}$ is not guaranteed to vanish. Indeed, due to the spectral misalignment, $\mathcal{E}_{\Delta}$ suffers from a leakage on the kernel of $\tilde{\Delta}_{\text{norm}}$, meaning it remains non-zero even as the signal smooths optimally relative to the dynamics. Consequently, observing $\mathcal{E}_{\Delta} \to 0$ in this regime implies a different phenomenon entirely: it requires the collapse of the embedding norms themselves to zero. This effectively conflates over-smoothing with over-shrinking. If one relies on $\mathcal{E}_{\Delta}$ without accounting for this spectral mismatch, the decay of the energy might be misidentified as a consequence of the contractive nature of the network layers, rather than correctly attributing the asymptotic behavior to the alignment with the Laplacian kernel. This distinction underscores the importance of consistency: to avoid misinterpreting spectral behavior, the chosen metric must commute with the operator governing the dynamics.\\
A critical ambiguity arises in the literature when over-smoothing is analyzed through a single scalar metric. For instance, \cite{arroyo2025vanishinggradientsoversmoothingoversquashing} attributes over-smoothing to the contractive nature of GNNs, observing that vanishing gradients drive the energy $\mathcal{E}_{\Delta}$ to zero. However, this perspective conflates two fundamentally different phenomena. To resolve this, we propose a strict decoupling based on systemic universality. We refer to spectral over-smoothing as the regime where, for arbitrary input graphs and generic signals, the output aligns with the low-frequency kernel of the underlying Laplacian while maintaining a strictly positive norm. Conversely, over-shrinking denotes the unconditional collapse of the embedding magnitudes to zero, irrespective of the graph topology or signal initialization. By monitoring only the energy without these universal constraints, one risks misinterpreting the trivial null-solution as spectral convergence.

\section{Numerical Examples}

In this section, we present a series of empirical experiments designed to illustrate the theoretical points discussed in the previous sections. We aim to show that over-smoothing conditions with respect to the normalized Laplacian that nothing have to do with over-shrinking, remain invisible when measured through standard proxies such as the Dirichlet energy or the Frobenius norm of the embeddings. Providing concrete evidence that these states are only identified by the compatible metric allows to disentangle over-smoothing from simple signal decay.\footnote{The source code to reproduce the experiments is available at \url{https://github.com/annabison/dirichlet-oversmoothing-measures}.}

\subsection{Non Over-shrinking GCNs Over-smooth}

We present empirical evidence showing that the information loss resulting from spectral collapse is distinct from the decay measured by the standard Dirichlet energy. Specifically, we isolate regimes where the dynamics are incompatible with $\Delta$ and strictly non-contractive: in these cases, the signal converges to a smooth state with non-zero embedding norms, a phenomenon that standard metrics fail to capture.
To observe the isolated effect of the GCN propagation operator, we conduct at first an experiment focused purely on signal smoothing. This experiment eliminates all learnable parameters (i.e., weight matrices $W$).\\
We select two distinct graphs from the ENZYMES dataset: one graph exhibiting a regular structure and one non-regular graph.
We then simulate a deep linear GCN by iteratively applying only the standard propagation operator for $K=50$ layers. The state at layer $k$ is thus defined by:
\[ 
X^{(k)} = A_{\text{norm}} X^{(k-1)} = A_{\text{norm}}^k X^{(0)}, 
\]
where $A_{\text{norm}} = D^{-1/2}AD^{-1/2}$.\\
At each layer $k \in [0, K]$, we track the same three metrics:
the Frobenius Norm ($\|X^{(k)}\|_F$), the standard Dirichlet energy ($\mathcal{E}_{\Delta}(X^{(k)}) = \frac{1}{|\mathcal{V}|} \mathrm{tr}\!\left( X^{(k),\top} \Delta X^{(k)} \right)$), and the normalized Dirichlet energy ($\mathcal{E}_{\Delta_\text{norm}}(X^{(k)}) = \frac{1}{|\mathcal{V}|} \mathrm{tr}\!\left( X^{(k),\top} \Delta_\text{norm} X^{(k)} \right)$). The objective is to directly compare the convergence behavior and smoothing properties of the $A_\text{norm}$ operator, contrasting its effect on graphs with regular versus non-regular topologies.
Results, computed in double precision, are shown in Figures \ref{fig:overshrinking_1} and \ref{fig:overshrinking_2}: for the non-regular graph, $\|X^{(k)}\|_F$ converges to a non zero value around $10$, while $\mathcal{E}_{\Delta_\text{norm}}$ decreases faster than $\mathcal{E}_{\Delta}$, in line with theoretical expectations that predict a convergence on each dimension to a graph signal $\vec{x} \; | \; \exists \; c \in \mathbb{R} \; |\; x_i =c\sqrt{d_i}$. \\
The plot reflects the nature of low-pass filter of the simple GCN without activations: it is a spectral function of the normalized Laplacian $A_{\text{norm}}= I-\Delta_\text{norm}$, hence having eigenvalues in $[-1,1]$ that dumps after several applications to an embedding matrix $X$, apart that the components that live in the kernel of $\Delta_\text{norm}$ that coincide with the eigen-space of $A_{\text{norm}}$ of eigenvalue $1$. The convergence to the kernel of $\Delta_\text{norm}$ does not imply the collapse of $\|X^{(k)}\|_F$ to zero, as the plot shows.
The regular graph in \ref{fig:overshrinking_2} has $\mathcal{E}_{\Delta_\text{norm}}$ that converges to a constant value, the theoretical expectations are that the two energies should both tend asymptotically to zero, this happens fast and values decrease as fast that after less than 20 layers there is a pattern caused by numerical instabilities. This result shows a way in which over-smoothing can occur without over-shrinking, i.e. without the convergence of the features to zero, highlighting how misleading could be to reduce the former to the latter.  

\begin{figure}[h!]
\centering 

\begin{minipage}{0.49\textwidth}
    \centering
    \includegraphics[width=\textwidth]{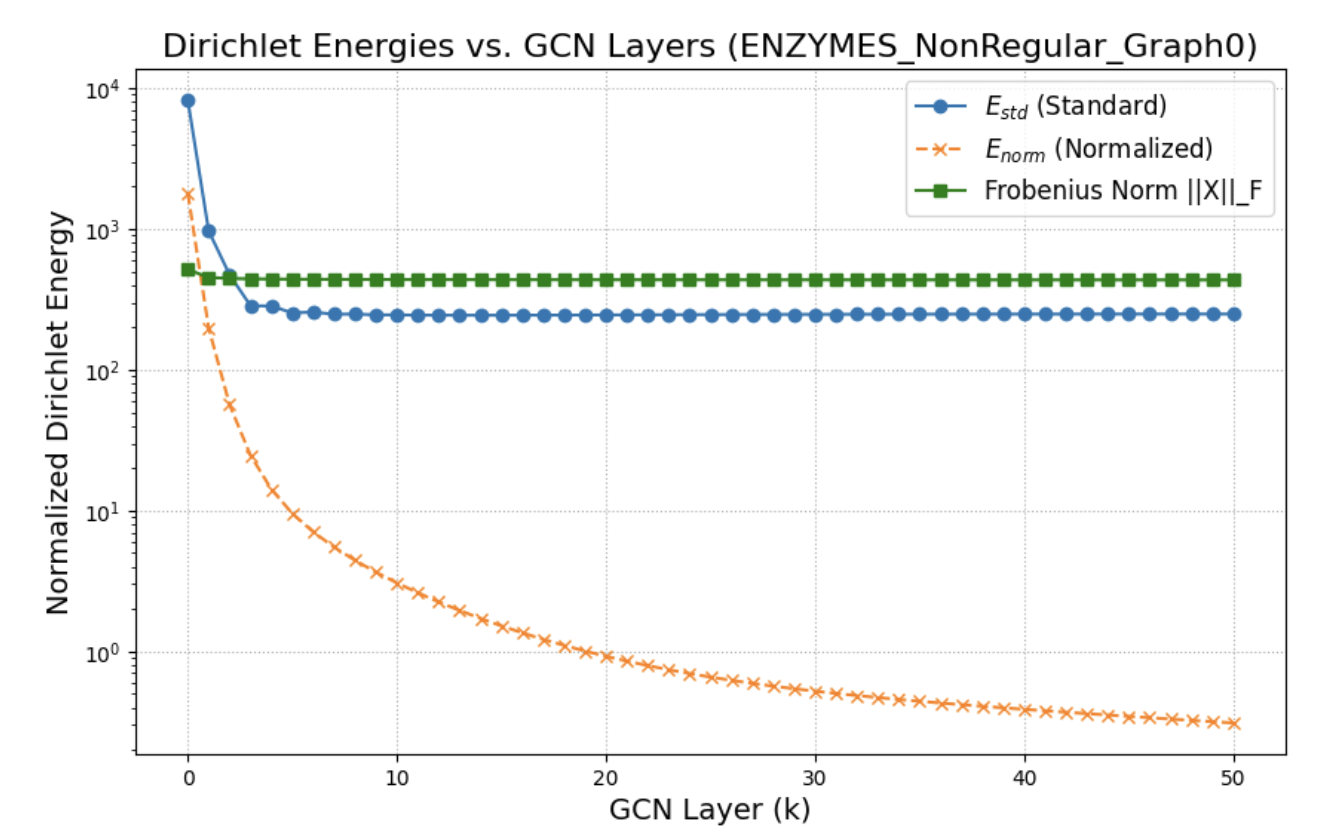} 
    \caption{Evolution of Dirichlet energies in logarithmic scale and the Frobenius norm of embeddings across 50 GCN layers. The experiment is conducted on the first graph from the ENZYMES dataset, a non-regular graph with 37 nodes, an average degree of 4.54, and a degree variance of 0.98.}
    \label{fig:overshrinking_1} 
\end{minipage}
\hfill 
\begin{minipage}{0.49\textwidth}
    \centering
    \includegraphics[width=\textwidth]{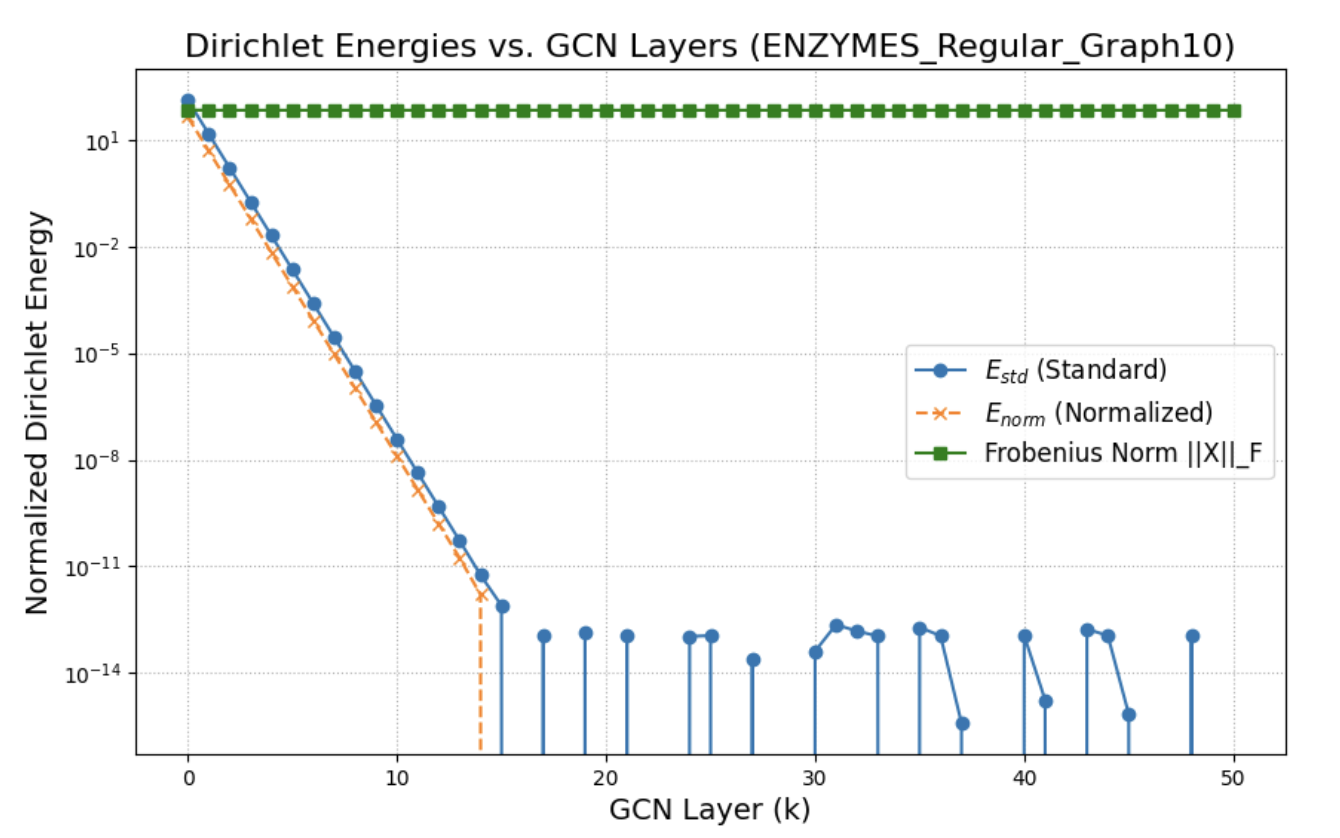} 
    \caption{Evolution of Dirichlet energies in logarithmic scale and the Frobenius norm of embeddings across GCN layers. The experiment is conducted on Graph 10 from the ENZYMES dataset, a regular graph with $4$ nodes of degree $3$.
      } 
    \label{fig:overshrinking_2} 
\end{minipage}

\end{figure}

\subsubsection{LCC of the Cora Dataset}
In the following experiment, we apply the same spectral filter used in the previous computation to a completely different graph: the Largest Connected Component (LCC) of the Cora dataset. We expect to observe the same trend of $\mathcal{E}_{\Delta_\text{norm}}$ as in the previous experiment, as the GCN implements a low-pass filter w.r.t. $\Delta_\text{norm}$ regardless of the specific input signal.
To empirically isolate the smoothing properties of the GCN propagation operator, we use a $K$-layer model structured to separate the initial feature transformation from the subsequent graph-based propagation.\\
In order to reduce the number of input features, rows of $ X^{(0)} \in \mathbb{R}^{|\mathcal{V}| \times m}$ are first projected by a single, bias-free linear layer\footnote{the weight matrix $W \in \mathbb{R}^{F \times C}$ is initialized using the default PyTorch method for `torch.nn.Linear`, which is a Kaiming uniform distribution \cite{He2015}.}: $X^{(1)} = X^{(0)}W$.\\
Subsequent layers $k = 1, \dots, K-1$ apply only the GCN propagation operator, without any trainable weights or non-linearities:
$ X^{(k+1)} = A_{\text{norm}} X^{(k)} $.\\
We execute this experiment on the LCC of the Cora dataset. At each layer $k = 0, \dots, K$, we track the evolution of the Frobenius norm of the embeddings, $\|X^{(k)}\|_F$; the standard Dirichlet energy, $\mathcal{E}_{\Delta}(X^{(k)})$; and the normalized Dirichlet energy, $\mathcal{E}_{\Delta_\text{norm}}(X^{(k)})$.
Results are shown in Figure \ref{fig:CoraNoW}. 

\begin{figure}[h!]
\centering
\includegraphics[width=0.7\textwidth]{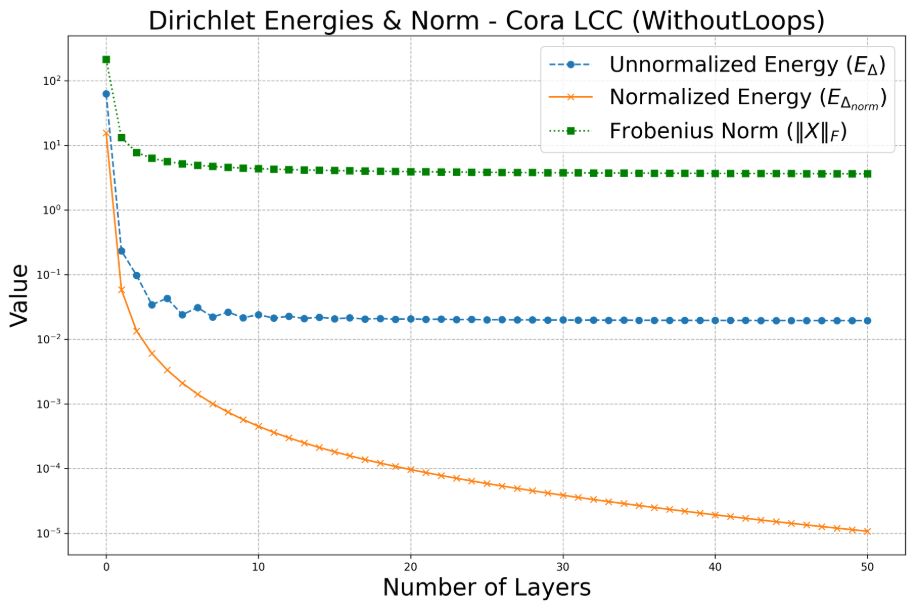}
\caption{Evolution in symmetric logarithmic scale of Dirichlet energy (normalized and unnormalized) and Frobenius norm of the embeddings across 50 layers of a GCN without weight matrices in all the layers apart the first one. The experiment uses a custom GCN model on the Largest Connected Component (LCC) of the Cora dataset. The rapid decay of the normalized Dirichlet energy, highlights the over-smoothing phenomenon, while the convergence of $\|X^{(k)}\|_F$ to a non-zero value excludes over-shrinking.}
\label{fig:CoraNoW}
\end{figure}

\subsubsection{Note on Weight Matrices}

If a graph is given in input to a GCN with all the weight matrices, performing for each layer the operation $X^{(k+1)} = A_{\text{norm}}X^{(k)}W^{(k+1)}$, in general what is obtained is the transformation $X^{(K)} = A_{\text{norm}}^KX^{(0)}W^{(1)}\dots W^{(K)}$, in particular the weight matrices $W^{(l)} \in \mathbb{R}^{m_{l-1} \times m_{l}}$ are different for each layer. The weight matrices $W^{(l)}$ to each layer $l$ of the GNN do not alter the overall effect of low-pass filter given by the operator $\hat{A}$: indeed it's straightforward to see that from associativity
\[ 
X^{(K)} = A_{\text{norm}} X^{(K-1)}W^{(K)} = A_{\text{norm}}^K X^{(0)} W^{(1)}\dots W^{(K)}=A_{\text{norm}}^K X^{\prime(0)}
\]
hence obtaining at the end of the last iteration the same output of applying a GCN without weight matrices for $k$ times to a different initial input $X^{\prime(0)}= X^{(0)} W^{(1)}\dots W^{(K)}$: being a low pass filter, the simple GCN will make graph signals collapse to the kernel of $\Delta_\text{norm}$ independently from the input.

\subsection{Indeterminacy of Spectral Filtering}

The goal of this subsection is to empirically verify that a GNN filter defined as a spectral function of $\Delta_{\text{norm}}$ does not behave as a consistent spectral filter with respect to the frequencies of $\Delta$ on non-regular graphs. A GCN layer implementing the operator $A_{\text{norm}} = I - \Delta_{\text{norm}}$, i.e. a polynomial of the normalized Laplacian, will be applied to a regular graph and to a non regular graph. According to the theoretical results of the previous sections, for a regular graph of average degree $d$ it will be $D = dI$, and hence the ratio among the two Dirichlet energies (for $X$ where it is well defined) will be $\mathcal{E}_{\Delta}(X)/\mathcal{E}_{\Delta_{\text{norm}}}(X)=\frac{\text{tr}(X^T(D-A)X)}{\text{tr}(X^Td^{-1/2}(D-A)d^{-1/2}X)}=d$, i.e. a constant independent from $X$, in particular, the same constant through all the layers of a GCN. Hence, the expected plot of the ratio through the layers of any shape will be then an horizontal line of value $d$.\\
For a non-regular graph, the ratio $\mathcal{E}_{\Delta}(X)/\mathcal{E}_{\Delta_{\text{norm}}}(X)=\frac{\text{tr}(X^T(D-A)X)}{\text{tr}(X^TD^{-1/2}(D-A)D^{-1/2}X)}$ depends on $X$ and evolves through the layers. This behavior is numerically illustrated in Figure \ref{fig:ratio}: for the regular graph (where $d=3$) in \ref{fig:ratioG10}, the ratio remains perfectly constant for the initial iterations. The instability observed after layer $k \approx 13$ is purely a numerical artifact: as verified in Figure \ref{fig:overshrinking_2}, beyond this point both energy values decay to magnitudes close to machine precision limits. Consequently, their ratio becomes ill-conditioned and dominated by floating-point errors.

\begin{figure}[h!]
\centering

\begin{subfigure}{0.32\textwidth}
    \centering
    \includegraphics[width=\textwidth]{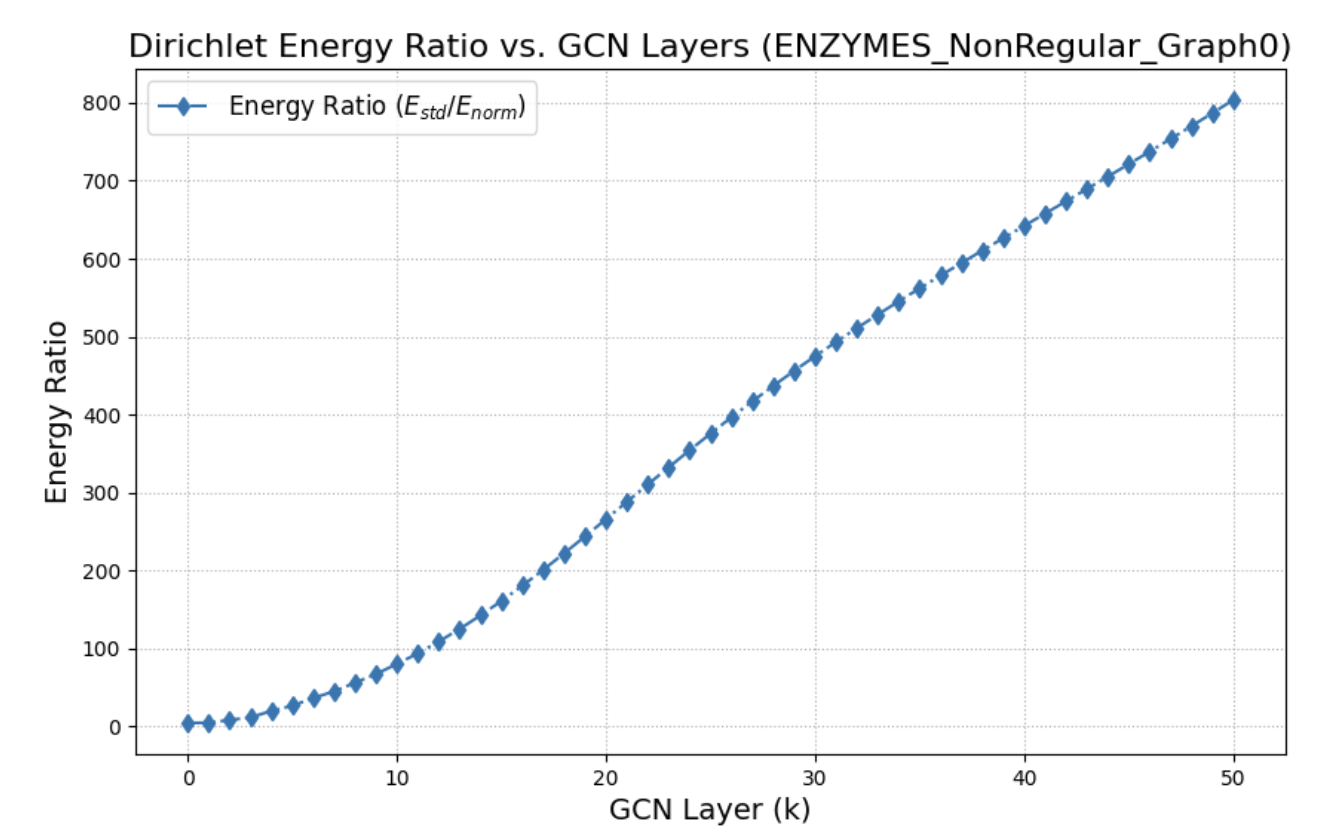}
    \caption{37 nodes, 168 edges, non regular.}
    \label{fig:ratioG0}
\end{subfigure}
\hfill 
\begin{subfigure}{0.32\textwidth}
    \centering
    \includegraphics[width=\textwidth]{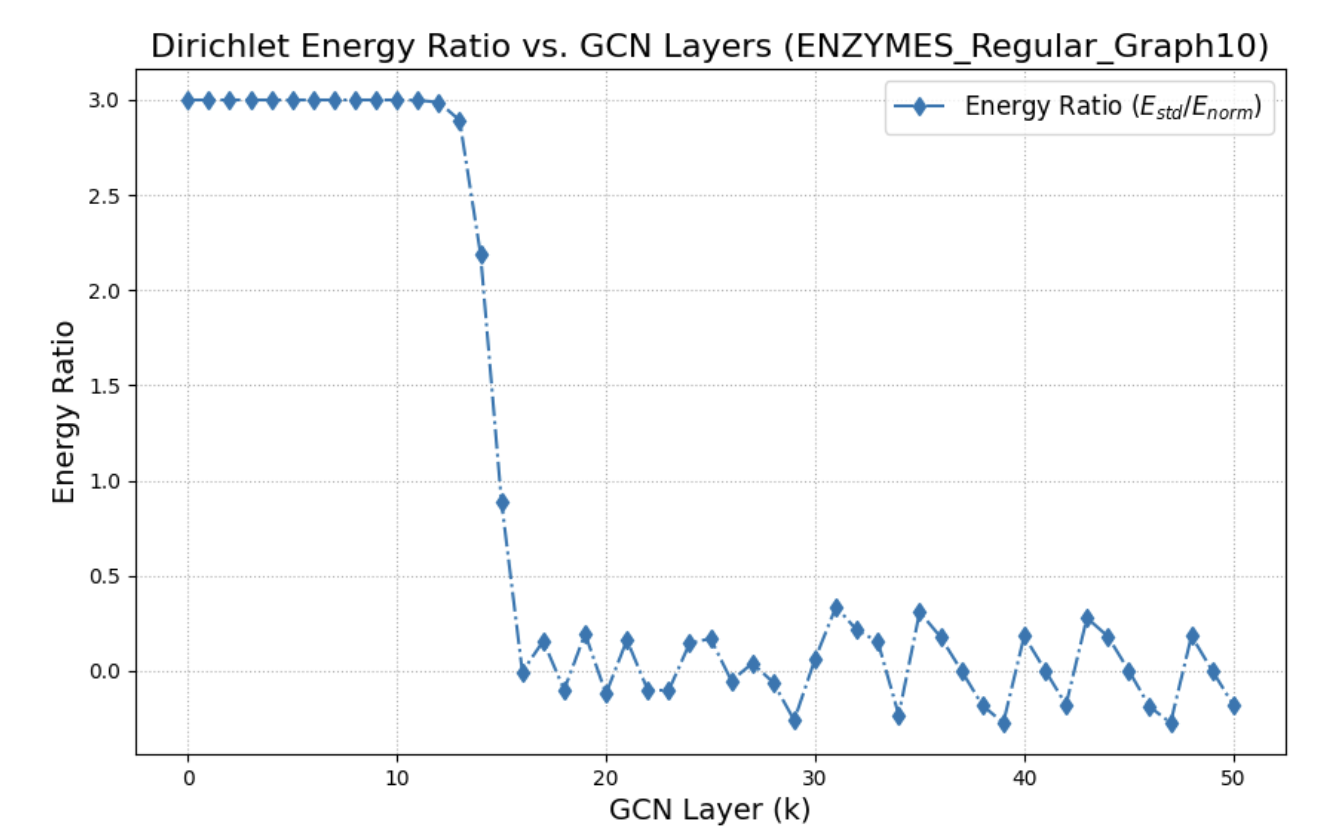}
    \caption{4 nodes, 12 edges, regular.}
    \label{fig:ratioG10}
\end{subfigure}
\hfill 
\begin{subfigure}{0.32\textwidth}
    \centering
    \includegraphics[width=\textwidth]{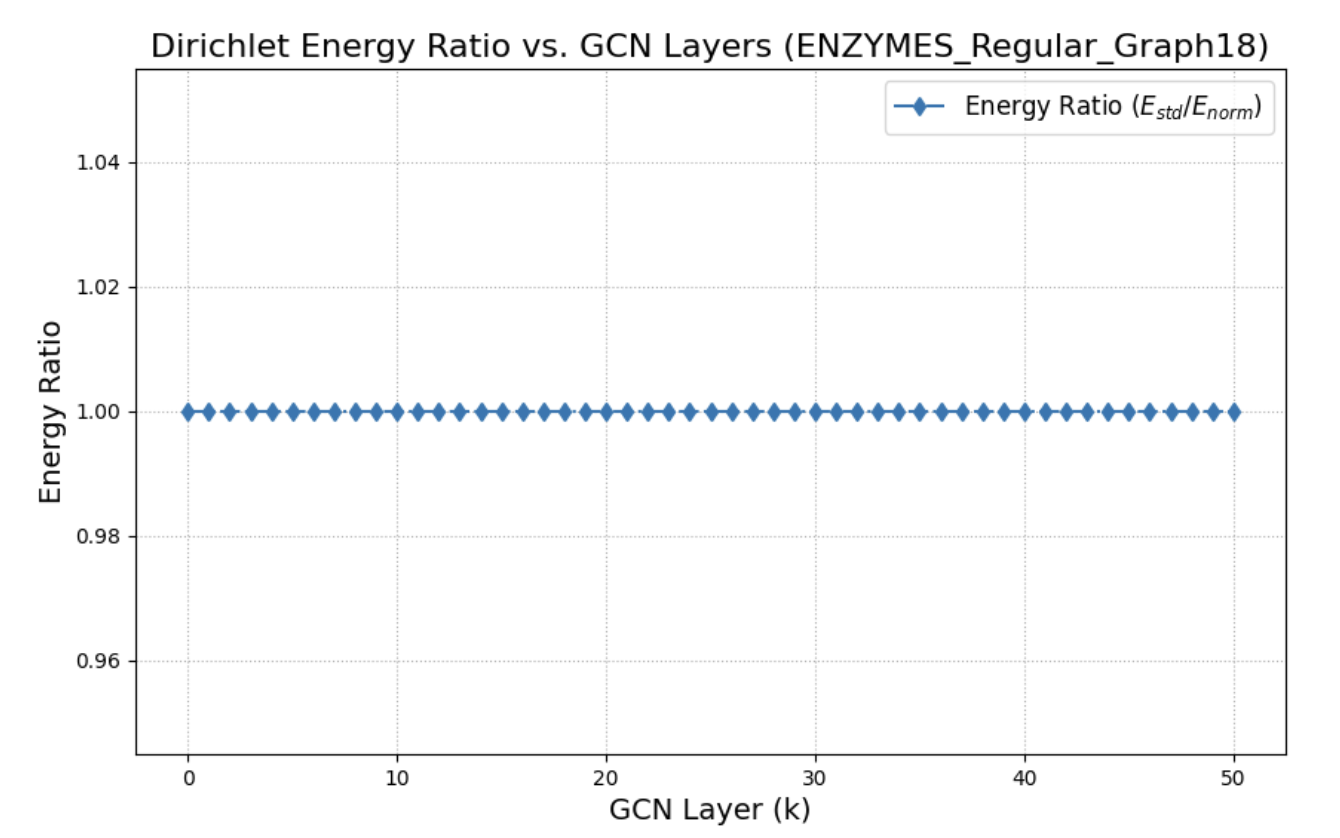}
    \caption{2 nodes, 1 edge, regular.}
    \label{fig:ratioG18}
\end{subfigure}

\caption{Evolution of the ratio $\mathcal{E}_{\Delta}(X)/\mathcal{E}_{\Delta_{\text{norm}}}(X)$ through the layers of a GCN implementing $X^{(k+1)} = A_{\text{norm}}X^{(k)}W^{(k)}$ for two graphs of the dataset ENZYMES.}
\label{fig:ratio}
\end{figure}

\FloatBarrier 

\section{Conclusions}
In this paper, we proposed a formal analysis of the spectral properties of the standard Laplacian $\Delta$ and the normalized Laplacians $\Delta_{\text{norm}}$ and $\tilde{\Delta}_{\text{norm}}$, illustrating their intrinsic differences within the framework of non-commuting operators. We highlighted that overlooking these distinctions can lead to misleading conclusions. Taking this analysis into account is crucial for future attempts to axiomatize over-smoothing, as our work exposed the pathologies inherent in a previous axiomatic proposal. In light of our analysis, general assertions such as those presented in Section 4.1 of \cite{arroyo2025vanishinggradientsoversmoothingoversquashing}, which attribute over-smoothing entirely to artifacts of signal collapse, thereby dismissing the distinct mechanism of convergence to a rank-one subspace, require critical reassessment. We conclude by underscoring that the validity of Dirichlet energy as an over-smoothing measure is strictly limited to GNNs implementing dynamics that are spectrally compatible with the Laplacian inducing the metric.

\section*{Statement on the Use of AI Tools}

The authors acknowledge the use of GitHub Copilot and Google Gemini to assist in the preparation of this manuscript. These tools were employed for generating boilerplate text for the experimental section, drafting and refining the experimental code, and improving the linguistic clarity and flow of the paper. The authors reviewed and edited the generated content and take full responsibility for the accuracy and integrity of the work presented.

\bibliographystyle{unsrt}
\bibliography{bibliography}

@misc{kipf2017semisupervisedclassificationgraphconvolutional,
      title={Semi-Supervised Classification with Graph Convolutional Networks}, 
      author={Thomas N. Kipf and Max Welling},
      year={2017},
      eprint={1609.02907},
      archivePrefix={arXiv},
      primaryClass={cs.LG},
      url={https://arxiv.org/abs/1609.02907}, 
}

@article{Rusch2023ASO,
  title={A Survey on Oversmoothing in Graph Neural Networks},
  author={T. Konstantin Rusch and Michael M. Bronstein and Siddhartha Mishra},
  journal={ArXiv},
  year={2023},
  volume={abs/2303.10993},
  url={https://api.semanticscholar.org/CorpusID:257632346}
}

@article{Taubin1995ASP,
  author={Taubin, G.},
  title={A signal processing approach to fair surface design},
  journal={Proceedings of the 22nd annual conference on Computer graphics and interactive techniques},
  year={1995},
}

@INPROCEEDINGS{466848,
  author={Taubin, G.},
  booktitle={Proceedings of IEEE International Conference on Computer Vision}, 
  title={Curve and surface smoothing without shrinkage}, 
  year={1995},
  volume={},
  number={},
  pages={852-857},
  doi={10.1109/ICCV.1995.466848}}

@article{Cai2020ANO,
  title={A Note on Over-Smoothing for Graph Neural Networks},
  author={Chen Cai and Yusu Wang},
  journal={ArXiv},
  year={2020},
  volume={abs/2006.13318},
  url={https://api.semanticscholar.org/CorpusID:220042028}
}

@article{DBLP:journals/corr/abs-1810-00826,
  author       = {Keyulu Xu and
                  Weihua Hu and
                  Jure Leskovec and
                  Stefanie Jegelka},
  title        = {How Powerful are Graph Neural Networks?},
  journal      = {CoRR},
  volume       = {abs/1810.00826},
  year         = {2018},
  url          = {http://arxiv.org/abs/1810.00826},
  eprinttype    = {arXiv},
  eprint       = {1810.00826},
  bibsource    = {dblp computer science bibliography, https://dblp.org}
}

@misc{zhang2025rethinkingoversmoothinggraphneural,
      title={Rethinking Oversmoothing in Graph Neural Networks: A Rank-Based Perspective}, 
      author={Kaicheng Zhang and Piero Deidda and Desmond Higham and Francesco Tudisco},
      year={2025},
      eprint={2502.04591},
      archivePrefix={arXiv},
      primaryClass={cs.LG},
      url={https://arxiv.org/abs/2502.04591}, 
}

@InProceedings{spectralinterpretation,
author="Bison, Anna
and Sperduti, Alessandro",
title="A Spectral Interpretation of Redundancy in a Graph Reservoir",
booktitle="Artificial Neural Networks and Machine Learning. ICANN 2025 International Workshops and Special Sessions",
year="2026",
publisher="Springer Nature Switzerland",
address="Cham",
pages="201--212",
isbn="978-3-032-04552-2"
}

@misc{arroyo2025vanishinggradientsoversmoothingoversquashing,
      title={On Vanishing Gradients, Over-Smoothing, and Over-Squashing in GNNs: Bridging Recurrent and Graph Learning}, 
      author={Álvaro Arroyo and Alessio Gravina and Benjamin Gutteridge and Federico Barbero and Claudio Gallicchio and Xiaowen Dong and Michael Bronstein and Pierre Vandergheynst},
      year={2025},
      eprint={2502.10818},
      archivePrefix={arXiv},
      primaryClass={cs.LG},
      url={https://arxiv.org/abs/2502.10818}, 
}

@misc{He2015,
      title={Delving Deep into Rectifiers: Surpassing Human-Level Performance on ImageNet Classification}, 
      author={Kaiming He and Xiangyu Zhang and Shaoqing Ren and Jian Sun},
      year={2015}, 
      eprint={1502.01852},
      archivePrefix={arXiv},
      primaryClass={cs.CV},
      url={https://arxiv.org/abs/1502.01852}, 
}

\end{document}